\documentclass[runningheads]{llncs}
\usepackage[table,dvipsnames]{xcolor}  
\usepackage[T1]{fontenc}
\usepackage{times}
\usepackage{mathtools}
\usepackage{amsmath, amssymb}
\usepackage{stmaryrd}
\usepackage{float}
\usepackage{hyperref}
\usepackage{cite}
\usepackage{microtype}
\usepackage{subcaption}
\usepackage{booktabs}
\usepackage[scaled=0.9]{inconsolata}
\usepackage{thmtools, thm-restate}
\usepackage{caption}
\usepackage{tikz}
\usepackage{multicol}
\usepackage{orcidlink}
\usepackage{tabularx}
\usepackage[nameinlink]{cleveref}
\usepackage{enumitem}
\usepackage{dirtytalk}
\usepackage{wrapfig}
\usepackage{hyphenat}
\usepackage{rotating}                     

\usepackage[vlined,algoruled,algo2e]{algorithm2e}
\DontPrintSemicolon
\crefname{algocf}{Algorithm}{Algorithms}
\Crefname{algocf}{Algorithm}{Algorithms}

\newcommand{\fcontext}{\mathbb{K} = ( G,M,I )}

\newcommand{\rcontext}{\mathbb{R} = ( G,M,I,\mathsf{R} )}

\newcommand{\minDerivation}[1]{\underline{(#1)^\downarrow}}

\renewenvironment{proof}[1][Proof]{%
  \par\noindent\textbf{#1.} }{\hfill$\Box$\par}
\newcommand{\twiddle}{\mathrel|\joinrel\sim}
\newcommand{\ntwiddle}{\mathrel|\joinrel\not\sim}

\newcommand{\dentails}{\mid\hskip-0.40ex\approx}
\newcommand{\ndentails}{\mid\hskip-0.40ex\not\approx}

\title{Rational Inference in Formal Concept Analysis}
\titlerunning{KLM Reasoning in FCA}
\author{Lucas Carr\inst{1}\orcidlink{0000-0001-7464-8422} \and
Nicholas Leisegang\inst{1}\orcidlink{0000-0002-8436-552X} \and 
Thomas Meyer\inst{1}\orcidlink{0000-0003-2204-6969} \and Sergei Obiedkov\inst{2,3}\orcidlink{0000-0003-1497-4001}}
\authorrunning{Carr, L. et al.}
\institute{University of Cape Town and CAIR, South Africa \\
\email{\{lucas,tmeyer\}@airu.org.za, lsgnic001@myuct.ac.za}
\and TU Dresden, Germany \and Center for Scalable Data Analytics and Artificial Intelligence (ScaDS.AI) Dresden/Leipzig, Germany\\
\email{sergei.obiedkov@tu-dresden.de}}
\begin{document}
\maketitle              
\begin{abstract}
Defeasible conditionals are a form of non-monotonic inference which enable the expression of statements like ``if $\phi$ then normally $\psi$''. The KLM framework defines a semantics for the propositional case of defeasible conditionals by construction of a preference ordering over possible worlds. The pattern of reasoning induced by these semantics is characterised by consequence relations satisfying certain desirable properties of non-monotonic reasoning. In FCA, implications are used to describe dependencies between attributes. However, these implications are unsuitable to reason with erroneous data or data prone to exceptions. Until recently, the topic of non-monotonic inference in FCA has remained largely uninvestigated. In this paper, we provide a construction of the KLM framework for defeasible reasoning in FCA and show that this construction remains faithful to the principle of non-monotonic inference described in the original framework. We present an additional argument that, while remaining consistent with the original ideas around non-monotonic reasoning, the defeasible reasoning we propose in FCA offers a more contextual view on inference, providing the ability for more relevant conclusions to be drawn when compared to the propositional case.
\end{abstract}
\section{Introduction}
\label{section:introduction}
Formal Concept Analysis (FCA) \cite{ganter2024formal,ganter2016conceptual} is an application of lattice theory to the long-standing philosophical view of concepts as a dual between extension---what may be referred to as an instance of a concept---and intension---what meaning we ascribe to a concept. It turns out that this dualism is quite elegantly modelled by Galois connections over sets of objects and attributes. The strength of this modelling has made FCA a popular tool for Knowledge Representation and Reasoning, information retrieval, rule mining, and ontology engineering. The operator `$\rightarrow$' called an \textit{attribute implication} is used to describe a relationship between two sets of attributes such that $A \rightarrow B$ means ``all those objects which have attributes $A$ also have attributes $B$''. 

Attribute implications are satisfied only if they are consistent with all available information. As a result, any counter-example would prevent a relation between attributes from being recognised. Frequently, however, we may have error-prone data where counter-examples should not be discarded; otherwise, in certain domains, it is useful to express relationships which only partially hold. The canonical example in non-monotonic reasoning of penguins and flying birds is an illustration of this latter point: imagine a context with many objects which may have the attributes `bird' and `flying'; only one of these objects is a penguin, which is a `bird' but does not have the attribute `flying'. It seems appropriate to suggest that expression of the partial correspondence between the attributes `bird' and `flying' is useful, in spite of the penguin. 

\textit{Defeasible conditionals} \cite{KLM2002nonmonotonic,what-does-a-conditional-knowledge-base-entail} are a particular kind of non-monotonic inference relation such that $\texttt{bird} \twiddle \texttt{flies}$ expresses the defeasible information that `Typical birds fly, while non-typical birds may not'. In the original work, these conditionals are given a semantics by a preference-relation $\prec$ over possible worlds such that the previous conditional holds if \texttt{flies} is satisfied by every $\prec$-minimal world also satisfying \texttt{bird}. These semantics induce consequence relations satisfying particular properties argued to be a reasonable account of non-monotonic inference. In particular, \textit{preferential} and \textit{rational} consequence relations are discussed. 

In this paper, we introduce the `$\twiddle$' operator to formal concept analysis from its original setting in propositional logic. We show that the consequence relation induced by the FCA setting is faithful to its propositional counterpart and that an equivalent non-monotonic entailment relation can be defined. On the other hand, we present an argument that the FCA setting offers an intuitive restriction on possible worlds through the formal context. This restriction allows for non-monotonic reasoning, in accordance with rational consequence, which considers a more contextual world-view than the propositional case. 
\subsubsection{Outline of this Paper:}
In \Cref{section:preliminaries}, we provide some basic notions in FCA that are relevant to this work, as well as a more thorough discussion on non-monotonic reasoning and the KLM framework for propositional logic. In \Cref{section:defeasible-fca}, some work is done to make the setting of FCA appropriate for preferential and rational consequence relations. After this, we introduce a defeasible conditional to the logic and show that for each system (preferential and rational) our definitions are sound and complete. \Cref{section:rational-fca} then discusses a more applied approach to defeasible reasoning in FCA and defines a rational closure---a non-monotonic entailment relation---for FCA. \Cref{section:discussion} provides a more elaborate example and discussion of how rational closure in FCA differs from the propositional case. \Cref{section:related-works} is a discussion of related work, while \Cref{section:conclusion} concludes and discusses future work.

\section{Preliminaries}
\label{section:preliminaries}
In the following, we make reference to three distinct languages: propositional logic, FCA's \textit{attribute logic}, and \textit{contextual concept logic}, an extended version of attribute logic. For the setting of propositional logic, lower-case Latin letters are used to denote propositional atoms, while Greek letters denote formulae in the language $\mathcal{L}$ defined in the usual way by $\alpha :: \top \mid \bot \mid a \mid \neg \alpha \mid \alpha \wedge \alpha \mid \alpha \vee \alpha \mid \alpha \rightarrow \alpha \mid \alpha \leftrightarrow \alpha$. The set of all valuations is denoted $\mathcal{U}$; when a valuation $v \in \mathcal{U}$ satisfies a propositional formula $\alpha$, we write $v \Vdash \alpha$. 

In FCA, subsets of objects (\textit{resp.} attributes) are denoted by capital letters, so $A \subseteq G$ (\textit{resp.} $B \subseteq M$). Lowercase letters are used to denote individual elements of these sets; to avoid confusion with propositional variables, membership will be made explicit. We return to Greek letters for the language of compound attributes, and so $\psi$ denotes a compound attribute. 
\subsection{Formal Concept Analysis}
\label{subsection:formal-concept-analysis}
Formal Concept Analysis (FCA) \cite{ganter2024formal,ganter2016conceptual} is a framework for reasoning about conceptual structures in data. We introduce only the basic notions relevant to this work.  The typical starting point in FCA, \textit{a formal context}, is a triple $\fcontext$ comprised of a set of objects $G$, a set of attributes $M$, and a binary relation $I \subseteq G \times M$. An object-attribute pair in the relation, $( g,m) \in I$, is interpreted as saying that \textit{``object $g$ has attribute $m$''}. 

Two maps, or \textit{derivation operators}, provide a way of describing the set of attributes (\textit{resp.} objects) which certain subsets of objects (\textit{resp.} attributes) have in common. They are given as $(\cdot)^\uparrow : 2^G \to 2^M$ and $(\cdot)^\downarrow : 2^M \to 2^G$ so that, for $A \subseteq G$ and $B \subseteq M$, 

\[
    \begin{aligned}
        A^\uparrow &= \{ m \in M \mid \forall g \in A, \ (g,m) \in I \}, \\
        B^\downarrow &= \{ g \in G \mid \forall m \in B, \ (g,m) \in I \}.
    \end{aligned}
\]

An \emph{attribute implication} over $M$ is a pair of attribute sets, denoted $A\rightarrow B$, where $A$ is the \emph{premise} and $B$ is the \emph{conclusion}. The implication is \emph{satisfied} by a set of attributes $C$ if and only if $A \not \subseteq C$ or $B \subseteq C$. In this case, we write $C \models A\rightarrow B$. Satisfaction is generalised to a formal context $( G,M,I)$ if and only if for all $g\in G$ it holds that $\{g\}^\uparrow \models A \rightarrow B$. This is equivalently expressed by $A^\downarrow \subseteq B^\downarrow$ or $B \subseteq A^{\downarrow\uparrow}$. An implication $i$ over $M$ is a \emph{logical consequence} of a set of implications $\mathcal{T}$ if and only if, for every subset $X \subseteq M$, if $X \models \mathcal{T}$ (meaning $X \models j$ for all $j \in \mathcal{T}$) then $X \models i$; we then write $\mathcal{T}\models i$. 
The Armstrong axioms are a syntactic counterpart to semantic entailment of attribute implications \cite{armstrong1974dependency}. 

\subsection{Defeasible Reasoning}
\label{subsection:defeasible-reasoning}
Non-monotonic reasoning is concerned with developing formal reasoning processes whereby a conclusion drawn under a premise can be withdrawn under the addition of another premise. More precisely, defeasible reasoning aims to express notions that while normal things behave a certain way, there may be atypical things that do not. 

Multiple motivations for non-monotonic reasoning exist. The one we consider most pertinent for this work is that monotonicity prevents one from discovering (useful) inferences when there are exceptions to the rule.  

Kraus et al. \cite{KLM2002nonmonotonic} introduce a collection of consequence relations for non-monotonic reasoning in propositional logic, which satisfy certain properties---considered to represent a good account of non-monotonic reasoning---called the \textit{rationality postulates}. We use $\twiddle_X$ to denote a defeasible consequence relation, where the subscript indicates which postulates the relation satisfies. A \textit{preferential consequence relation}, denoted $\twiddle_P$, is one that satisfies the postulates below.  
\begin{table}[H]
    \begin{center}
    \begin{tabularx}{0.8\textwidth}{lXr}
        \textbf{(REF)} & Reflexivity              & $\phi \twiddle \phi$ \\ 
        \textbf{(LLE)} & Left Logical Equivalence & $\phi \equiv \psi$ and $\psi \twiddle \gamma$ implies $\phi \twiddle \gamma$\\
        \textbf{(RW)}  & Right Weakening          & $\psi \rightarrow \gamma$ and $\phi \twiddle \psi$ implies $\phi \twiddle \gamma$ \\
        \textbf{(AND)} & And                      & $\phi \twiddle \psi$ and $\phi \twiddle \gamma$ implies $\phi \twiddle \psi \land \gamma$ \\
        \textbf{(OR)}  & Or                       & $\phi \twiddle \gamma$ and $\psi \twiddle \gamma$ implies $\phi \lor \psi \twiddle \gamma$ \\
        \textbf{(CUT)} & Cut                      & $\phi \land \psi \twiddle \gamma$ and $\phi \twiddle \psi$ implies $\phi \twiddle \gamma$ \\ 
        \textbf{(CM)}  & Cautious Monotonicity        & $\phi \twiddle \psi$ and $\phi \twiddle \gamma$ implies $\phi \land \psi \twiddle \gamma$
    \end{tabularx}
    \end{center}
\end{table}
\vspace{-2em}
Preferential consequence relations are given a semantics by construction of preferential interpretations.

\begin{definition}[Preferential Interpretation]
    \label{definition:preferential-interpretation}
    A \emph{preferential interpretation} $\mathcal{P}$ is a triple $(S, l, \prec)$ with a set of states $S$, a function $l:S \to \mathcal{U}$ mapping states to valuations, and a strict partial order $\prec$ on $S$. $\mathcal{P}$ satisfies $\phi \twiddle \psi$ if and only if the $\prec$-minimal states in the set of states satisfying $\phi$, denoted $\llbracket \phi \rrbracket$, also satisfy $\psi$. 
\end{definition}
The stronger system of \textit{rational consequence relations} ($\twiddle_R$) is introduced in \cite{what-does-a-conditional-knowledge-base-entail}, which satisfies all the above postulates with the addition of
\vspace{-1em}
\begin{table}[H]
    \begin{center}
    \begin{tabularx}{0.8\textwidth}{lXr}
        \textbf{(RM)} & Rational Monotonicity         & $\phi \twiddle \psi$ and $\phi \ntwiddle \neg \gamma$ implies $\phi \land \gamma \twiddle \psi$ 
    \end{tabularx}
    \end{center}
\end{table}
\vspace{-2.5em}
It is obvious that rational consequence relations are a type of preferential relations. The additional property (RM) is read as `If, under normal circumstances, the presence of $\phi$ implies $\psi$ but not the negation of $\gamma$, then I should be able to take $\gamma$ as true without causing the retraction of $\psi$'. Rational consequence relations are given a semantics by means of construction of a preferential structure on interpretations, called \textit{ranked interpretations} \cite{what-does-a-conditional-knowledge-base-entail,casini2019taking}.  
\begin{definition}[Ranked Interpretation]
    A \emph{ranked interpretation} $\mathcal{R}$ is a preferential interpretation $(S, l, \prec)$ where there exists a ranking function $\Omega : S \to \mathbb{N} \cup \{\infty\}$ that is \emph{convex}:
    for every $i \in \mathbb{N}$, if there exists some $s \in S$ with $\Omega(s) = i > 0$, then there exists some $t \in S$ such that $\Omega(t) = 0 \leq j < i$. Then, $s \prec t$ if and only if $\Omega(s) < \Omega(t)$ where $<$ is the usual strict total order on $\mathbb{N} \cup \{\infty\}$.  
\end{definition}

The distinction between preferential and rational consequence is discussed in more depth in \Cref{subsection:defeasible-conditionals-fca}. The convexity property for a ranked interpretation ensures that, when we consider a ranked interpretation as series of strata, there are no empty ranks. The satisfaction of a defeasible conditional by a ranked interpretation follows the same operation as with preferential interpretations where we consider only the $\prec$-minimal valuations. We recall the soundness and completeness results for preferential and rational reasoning from original work \cite{KLM2002nonmonotonic,what-does-a-conditional-knowledge-base-entail}, which link preferential and rational consequence relations to preferential and ranked interpretations, respectively.

\begin{theorem}[Soundness]
    \label{theorem:propositional-soundness}
    If $\mathcal{P}$ is a preferential (\textit{resp.} $\mathcal{R}$ is a ranked) interpretation, the induced consequence relation $\twiddle_\mathcal{P}$ (\textit{resp.} $\twiddle_\mathcal{R}$) is preferential (\textit{resp.} rational). 
\end{theorem}

\begin{theorem}[Completeness]
    \label{theorem:propositional-completeness}
    A consequence relation is preferential (\textit{resp.} rational) if and only if it is defined by a preferential (\textit{resp.} ranked) interpretation. 
\end{theorem}

As a final point, we discuss adequate notions of non-monotonic entailment. Any entailment relation whereby all models of a knowledge base are considered represents a Tarskian notion of consequence, and remains monotonic. 

Lehmann and Magidor \cite{what-does-a-conditional-knowledge-base-entail} propose an approach where consideration is restricted to a subset of ranked interpretations. One method of selecting which subset should be considered is due to Giordano et al. \cite{giordano2015semantic}. There, it is shown that a preference relation $\preceq_\mathcal{R}$ can be constructed so that, if $\mathcal{R}_1$ and $\mathcal{R}_2$ are ranked interpretations satisfying the set $\Lambda$ of defeasible conditionals, then $\mathcal{R}_1 \preceq_\mathcal{R} \mathcal{R}_2$ if and only if $\mathcal{R}_1(u) \leq \mathcal{R}_2(u)$ for all valuations $v \in \mathcal{U}$. 

The argument for this preference is based on the \textit{presumption of typicality}: that, in the absence of evidence to the contrary, we regard things as typical \cite{what-does-a-conditional-knowledge-base-entail}. In the finite case, this preference relation has a minimum element denoted $\mathcal{R}_{RC}$ \cite{what-does-a-conditional-knowledge-base-entail}, which defines the \textit{rational closure} of $\Lambda$ (independently, Pearl \cite{pearlSystemZ} proposed \textit{System Z}, which turns out to be rational closure). This entailment relation is non-monotonic, as addition of new information may result in changes to the preference relation internal to each ranked interpretation and result in a different $\preceq_\mathcal{R}$-minimum ranked context. 

The \texttt{BaseRank} algorithm \cite{what-does-a-conditional-knowledge-base-entail,freund1998preferential} partitions the \textit{materialisation}---the process of turning a defeasible knowledge base into its classical counterpart---of a knowledge base into layers of exceptionality, where higher ranks indicate more exceptional statements. Already classical statements $\phi$ in the knowledge base are considered defeasible statements of the form $\phi \twiddle \bot$. These statements are assigned the infinite rank, indicating they should not be removed. 
\texttt{RCProp} checks entailment of a query from this partition: lower ranks are removed until the query is no longer considered exceptional, after which entailment is determined classically with respect to the remaining statements, alongside statements in the infinite rank. 
It was shown by Giordano et al. \cite{giordano2015semantic} that the ranked interpretation induced by \texttt{BaseRank} is the $\preceq_\mathcal{R}$-minimum element. 

\begin{tiny}
\begin{figure}
\noindent\begin{minipage}[t]{0.51\textwidth}
    \begin{algorithm2e}[H]
		      \caption{\texttt{BaseRank}}
            \label{algorithm:baserank}
		      \KwIn{A set of defeasible conditionals $\mathcal{K}$} 
		      \KwOut{An ordered tuple $(R_0,\ldots,R_{n-1},R_\infty,n)$}
		    $i := 0$ \;
            $E_0:=\{\phi\rightarrow\psi\mid \phi\twiddle\psi \in\mathcal{K}\}$ \;
            \While{$E_{i-1} = E_1$} {
                $E_{i+1} \coloneqq \{\phi \rightarrow \psi \in E_i \mid E_i \models \neg \phi \}$\; 
                $R_i \coloneqq E_i \setminus E_{i+1}$\;
                $i \coloneqq i + 1$\; }
            $R_\infty \coloneqq E_{i-1}$\;
            \eIf {{$E_{i-1} = \emptyset$}}{
                $n \coloneqq i - 1$\; }
            {
                $n \coloneqq i$; }
            \Return {$(R_0, \ldots, R_{n-1}, R_\infty, n)$}\;
    \end{algorithm2e}
\end{minipage}
\hfill
\begin{minipage}[t]{0.47\textwidth}
    \begin{algorithm2e}[H]
            \caption{\texttt{RCProp}}
            \label{algorithm:RCProp}
            \KwIn{A set of defeasible conditionals $\mathcal{K}$}
            \KwIn{ A defeasible conditional $\phi \twiddle \psi$}
            \KwOut{True if $\mathcal{K} \dentails_\texttt{RC} \phi \twiddle \psi$ otherwise False}
             $(R_0, \ldots, R_{n-1}, R_\infty, n) \coloneqq \texttt{BaseRank}(\mathcal{K})$\;
             $i \coloneq 0$\; 
             $R \coloneqq \bigcup\limits_{j=0}^{j < n} R_j$\; 
            \While {$R_\infty \cup R \models \neg \phi$ and $R \not = \emptyset$}{
                 $R \coloneq R \setminus R_i$\; 
                 $i \coloneq i+1$\;
             }
            \Return {$R_\infty \cup R \models \phi \rightarrow \psi$}\;
    \end{algorithm2e}
\end{minipage}
\end{figure}
\end{tiny}

\section{Defeasible Reasoning in FCA}
\label{section:defeasible-fca}

In the previous section we provided a brief account of the KLM approach to defeasible reasoning in propositional logic. This section proposes a relatively simple translation of these ideas into FCA. We take a scaffolded approach and first describe the setting required for preferential reasoning. It is then demonstrated why preferential consequence is undesirable, and how a small adjustment to the preference relation results in the target of rational inference. The majority of the proofs for theorems, propositions, and lemmas in this section are found in the Appendix of the full version of this paper. 

Before this begins, some necessary changes to the typical FCA setting require discussion. Perhaps the reader might have some suspicions regarding what a rational consequence relation might mean in FCA, given that certain postulates -- namely (OR) and (RM) -- require a logic which has negation and disjunction, the expressivity for which is beyond the attribute logic underlying FCA.
\subsection{Contextual Attribute Logic}
\label{subsection:compound-attributes}
In the usual setting, implications are defined over sets of attributes $\{m_1,\ldots,m_n\}$ with a conjunctive view. Essentially, implications can be considered analogous to conjunctions of definite Horn clauses which share the same negated literals \cite{ganter2016conceptual}. A result of this is that the expressivity for disjunction or negation on either side of the implication is not present. Moreover, FCA typically restricts the meaning of the incidence relation to positive information. Several attempts to extend the expressivity of FCA's attribute logic are well documented \cite{ganter2016conceptual,rodriguez2014negative,perez2021new,ganter1999contextual}, but a single approach is yet to be widely accepted.

We adopt the idea of \textit{compound attributes}, due to Ganter \cite{ganter1999contextual}. Compound attributes allow for the construction of more complex formulae defined by the extension of (normal) attributes and other compound attributes. The language, denoted $\mathcal{L}(M)$, is defined recursively by $\phi :: \{m\} \mid \neg \phi \mid \phi_1 \land \phi_2 \mid \phi_1 \lor \phi_2$. A satisfaction relation is defined as follows:
\begin{definition}[Compound Attributes]
\label{definition:compound-attributes}
	Given a context $( G,M,I)$, $\Vdash$ is a \emph{satisfaction relation} between objects and compound attributes in the language $\mathcal{L}(M)$, where $g \Vdash \phi$ is interpreted as $g \in \phi^\downarrow$
	\begin{itemize}
		\item[--] $g \Vdash \{m\}$ if and only if $g \in \{m\}^\downarrow$
		\item[--] $g \Vdash \neg \phi$ if and only if $g \in G \setminus \phi^\downarrow$
		\item[--] $g \Vdash \phi_1 \lor \phi_2$ if and only if $g \in \phi_1^\downarrow \cup \phi_2^\downarrow$
		\item[--] $g \Vdash \phi_1 \land \phi_2$ if and only if $g \in \phi_1^\downarrow \cap \phi_2^\downarrow$
	\end{itemize}
	where $g\in G$, $m \in M$, and $\phi \in \mathcal{L}(M)$.
\end{definition}
%

This definition allows for arbitrarily complex compound attributes to be derived from a base set $M$ of normal attributes, analogous to propositional atoms. In order to talk about which objects satisfy a more complex compound attribute, the attribute-to-object operator is extended to consider compound attributes: $(\cdot)^\downarrow : \mathcal{L}(M) \to 2^G$. Then $\phi^\downarrow \coloneqq \{g \in G \mid g \Vdash \phi\}$. To determine whether an implication $\phi \rightarrow \psi$ over $\mathcal{L}(M)$ is satisfied by a context, it is then sufficient to show that $\phi^\downarrow \subseteq \psi^\downarrow$. 
\clearpage

\begin{example}
	In the context below, $\phi = \texttt{Rain} \lor \texttt{Wind}$ is a compound attribute with $\phi^\downarrow = \{ \texttt{Day 2, Day 3\}}$. The implication $\phi \rightarrow \texttt{Cold}$ holds in the context since $\phi^\downarrow \subseteq \texttt{Cold}^\downarrow$. 
    \vspace{1em}

	\noindent\begin{minipage}{\linewidth} 
		\centering
		\begin{tabular}{|c|c|c|c|c|>{\columncolor{teal!20}}c|}
			\hline
			               & \rotatebox{90}{\texttt{Sun}} & \rotatebox{90}{\texttt{Rain }} & \rotatebox{90}{\texttt{Wind }} & \rotatebox{90}{\texttt{Cold }} & \rotatebox{90}{$\phi$} \\
			\hline
			\multicolumn{6}{c}{} 
            \vspace{-0.9em}                                                                                                                                      \\
			\hline
			\texttt{Day 1} & $\times$                     &                                &                                &                                &                        \\
			\hline
			\texttt{Day 2} &                              &                                & $\times$                       & $\times$                       & $\times$               \\
			\hline
			\texttt{Day 3} &                              & $\times$                       &                                & $\times$                       & $\times$               \\
			\hline
			\texttt{Day 4} &                              &                                &                                & $\times$                       &                        \\
			\specialrule{1.25pt}{0pt}{0pt}
		\end{tabular}
		\label{tab:my_label}
	\end{minipage}

\end{example}

\subsection{Defeasible Conditionals over Compound Attributes}
\label{subsection:defeasible-conditionals-fca}
We now introduce the `$\twiddle$' as a connective over compound attributes. Then, $\phi \twiddle \psi$ is to be interpreted as ``The typical objects that satisfy the (compound) attribute $\phi$ also satisfy $\psi$''. Implicit in this reading is that a suitable structure for defining a semantics must have an explicit notion of when an object is `typical' and when it is not; i.e., a preference relation. It is then obvious that a formal context alone is insufficient. 

We introduce a \textit{preferential context} as an extension to the formal context, which provides such a relation on objects. \Cref{definition:minimisation} provides a useful notation for referring the minimal (preferred) objects from a set.  






\begin{definition}[Minimisation]
	\label{definition:minimisation}
    Let $A$ be a set equipped with a strict partial order $\prec$. The \emph{minimisation} of ${B}$ is a selection function $\underline{B} : 2^A \mapsto 2^A$ that maps each subset $B\subseteq A$ to the set of all $\prec$-minimal elements of $B$, i.e., 
    \[\underline{B} = \{g \in B \mid \nexists h \in B \text{ such that } h \prec g\}\]   
\end{definition}

We can then introduce a preferential context, as well as a notion for when a defeasible conditional is satisfied

\begin{definition}[Preferential Context]
	\label{definition:preferential-context}
	A \emph{preferential context} $\mathbb{P} = (G,M,I,\prec)$ is a quadruple where  $(G,M,I)$ is a formal context and $\prec$ is a strict partial order over the set of objects $G$ representing a preference relation. 
    
    A \textit{defeasible conditional} $\psi \twiddle \phi$ over the set $\mathcal{L}(M)$ of compound attributes is \emph{satisfied} by a preferential context $\mathbb{P}$ if and only if for each object $g \in \minDerivation{\psi}$ it is also the case that $g \Vdash \phi$; which can be expressed equivalently as $\minDerivation{\psi} \subseteq \phi^\downarrow$, or $\phi^\downarrow \subseteq \underline{(\psi)}^{\downarrow \uparrow}$.
\end{definition}

\begin{example}
	\label{example:preferential-context}
	If we restrict consideration to the classical context in \Cref{fig:context}, the implication $\texttt{Non-metal} \rightarrow \texttt{Gas}$ is not satisfied by the context because \texttt{Carbon} is a counter-example. Instead, we consider the preferential context where the order is given by the Hasse diagram in \Cref{fig:order}, and notice that the defeasible counterpart to the classical implication, $\texttt{Non-metal} \twiddle \texttt{Gas}$, is satisfied by the preferential context.

	\vspace{1em}

	\noindent\begin{minipage}{\linewidth} 
		\centering
		\begin{minipage}{0.5\linewidth}
			\centering
            \begin{small}
			\begin{tabular}{|l|l|l|l|l|l|l|}
				\hline
				                  & \rotatebox{90}{\texttt{Gas}}
				                  & \rotatebox{90}{\texttt{Non-metal }}
				                  & \rotatebox{90}{\texttt{Reactive}}
				                  & \rotatebox{90}{\texttt{Essential }}
				                  & \rotatebox{90}{\texttt{Solid}}
				                  & \rotatebox{90}{\texttt{Abundant}}                                                         \\
				\hline
				\multicolumn{6}{c}{} \vspace{-0.9em}                                                                          \\
				\hline
				\texttt{Helium}   & $\times$                           & $\times$ &          &          &          & $\times$ \\
				\hline
				\texttt{Hydrogen} & $\times$                           & $\times$ & $\times$ & $\times$ &          & $\times$ \\
				\hline
				\texttt{Carbon}   &                                    & $\times$ &          & $\times$ & $\times$ &          \\
				\specialrule{1.25pt}{0pt}{0pt}
			\end{tabular}
			\captionof{figure}{A context of elements}
			\label{fig:context}
            \end{small}
		\end{minipage}%
        \hfill
		\begin{minipage}{0.5\linewidth} 
			\centering
			\vspace{0em}
            \begin{small}
			\begin{tikzpicture}[
					scale=1,
					every node/.style={
							minimum size=0.5cm,
							inner sep=0pt,
						},
					node distance=1cm,
					rounded corners=5mm 
				]

				\draw[black] (-2,0) rectangle (2,3);

				\node (Bene Gesserit) at (1,1.5) {\small \texttt{Hydrogen}};
				\node (Spacing Guild) at (-1,1) {\small \texttt{Helium}};
				\node (Fremen) at (-1,2) {\small \texttt{Carbon}};

				\draw (Fremen) -- (Spacing Guild);

			\end{tikzpicture}
			
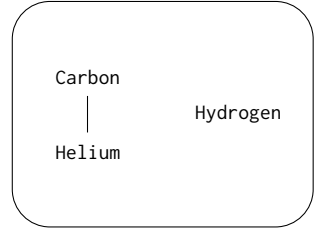
\captionof{figure}{Preference over objects}
			\label{fig:order}
            \end{small}
		\end{minipage}
	\end{minipage}
\end{example}

We let $\twiddle_\mathbb{P}$ refer to the consequence relation induced by the semantics from \Cref{definition:preferential-context} for defeasible conditionals. It should be quite clear that $(\phi,\psi) \in \twiddle_\mathbb{P}$ (or, $\phi \twiddle_\mathbb{P} \psi$) if and only if $\mathbb{P}$ satisfies $\phi \twiddle \psi$.

It remains to be shown that these defeasible conditions are a faithful translation of the original KLM framework. For our soundness result we now show that every relation $\twiddle_\mathbb{P}$ induced by a preferential context $\mathbb{P}$ is a preferential consequence relation. The following lemma is a reasonably obvious result of the semantics we have given to defeasible conditionals, and is helpful for the soundness result. 

\begin{restatable}{lemma}{Equivalence}
	\label{lemma1}
	If $\mathbb{P}$ is a preferential context, then it satisfies the defeasible conditional $\phi \twiddle \psi$ if and only if $\minDerivation{\phi} = \underline{\phi^\downarrow \cap \psi^\downarrow}$.
\end{restatable}


    

\begin{restatable}[Soundness of Preferential Contexts]{theorem}{SoundnessP}
	\label{theorem:soundness}
	For any preferential context $\mathbb{P}$, the induced relation $\twiddle_\mathbb{P}$ defines a preferential consequence relation, i.e., $\twiddle_\mathbb{P}$ is closed under (REF), (LLE), (RW), (AND), (OR), (CUT), and (CM).
\end{restatable}

We now show that for any preferential consequence relation $\twiddle_\mathcal{P}$, there exists a preferential context $\mathbb{P}$ such that the induced consequence relation $\twiddle_\mathbb{P}$ coincides exactly with $\twiddle_\mathcal{P}$. This yields our completeness result. Our approach relies on the completeness result for preferential interpretations from \cite{KLM2002nonmonotonic}, as recalled in \Cref{theorem:propositional-completeness}. 

As a preliminary step, we show that for any preferential interpretation $\mathcal{P}$ over the set $\mathcal{M}$ of propositional atoms, one can construct a preferential context $\mathbb{P}$ over the attribute set $M$ such that the consequence relation $\twiddle_\mathbb{P}$ induced by $\mathbb{P}$ corresponds to $\twiddle_\mathcal{P}$ via a translation between propositional formulae and compound attributes. In this sense, $\twiddle_\mathbb{P}$ reflects the pattern of reasoning described by $\twiddle_\mathcal{P}$. 

The translation is described by construction of a \textit{derived context}. To avoid ambiguity, we adopt the following notational convention: Greek letters $\phi, \psi, \ldots$ denote compound attributes in the language $\mathcal{L}(M)$, whereas $\phi_\mathcal{M}, \psi_\mathcal{M}, \ldots$ denote propositional formulae over the set $\mathcal{M}$ of propositional atoms.

\begin{definition}[Derived Context]
    \label{definition:interpretation-to-context}
    For a preferential interpretation $\mathcal{P} = ( S, l, \prec)$ over propositional atoms $\mathcal{M}$, the \emph{derived preferential context} is given by $\mathbb{P}^\mathcal{P} = (G, M, I, \prec_*)$
    where $G = \{g_s \mid s \in S\}$, 
    $M = \mathcal{M}$, 
    $(g_s,m) \in I$ if and only if $l(s) \Vdash m_\mathcal{M}$, and 
    $g_s \prec_* g_t$ if and only if $s \prec t$. 
\end{definition}

The following lemmas establish the link, in terms of consequence relations, between preferential interpretations and derived preferential contexts.  

\begin{restatable}{lemma}{Correspondence}
    \label{lemma:correspondence-propositional-compound}
    For any propositional formula $\phi_\mathcal{M}$, $l(s) \Vdash \phi_\mathcal{M}$ if and only if $g_s \Vdash \phi$. 
\end{restatable}

    

\begin{restatable}{lemma}{DefeasibleCorrespondence}
    \label{lemma:correspondence-conditionals-propositional-preferential}
    A defeasible implication $\phi_\mathcal{M} \twiddle \psi_\mathcal{M}$ over $\mathcal{M}$ is satisfied by a preferential interpretation $\mathcal{P}$ if and only if the compound attribute counterpart $\phi \twiddle \psi$ is satisfied in the derived preferential context $\mathbb{P}^\mathcal{P}$. 
\end{restatable}

While the consequence relation induced by a derived preferential context is not strictly equivalent to that of the corresponding propositional interpretation; \Cref{lemma:correspondence-propositional-compound} and \Cref{lemma:correspondence-conditionals-propositional-preferential} establish a systematic correspondence between the two. These results show that the inferential structure of a propositional consequence relation can be faithfully reflected in the compound attribute language of a preferential context.

    

\begin{restatable}[Completeness of Preferential Contexts]{theorem}{CompletenessP}
    \label{theorem:representation-preferential-contexts}
	A consequence relation is preferential if and only if it can be induced by some preferential context $\mathbb{P}$. 
\end{restatable}

    

We use this correspondence to show that completeness for preferential contexts directly follows from the completeness result of preferential interpretations in \cite{KLM2002nonmonotonic} (recalled in \Cref{theorem:propositional-completeness}).
%
%
\Cref{theorem:soundness} and \Cref{theorem:representation-preferential-contexts} provide a representation result for preferential consequence relations in FCA. However, preferential consequence relations allow for some unintuitive consequences. 

As an illustration, consider the following scenario: at a university, a typical student graduates. Suppose that at least one of these typical students---call them Alice---is also clever. Alice should then be considered a typical clever student (in fact, this follows from \Cref{lemma1}). Suppose another student, Bob, is not a typical student because, despite being clever, never works and will not graduate. It seems strange to accept that Bob might be a typical clever student: by virtue of Alice being a typical student, and Bob being an a-typical student, it seems reasonable to suggest that Alice is more typical than Bob and thus Bob should not be a typical clever student (since Alice is one too). 

But this not required in preferential reasoning; \Cref{example:preferential-context} demonstrates how certain preference relations allow for this exact scenario to unfold. One can easily verify the satisfaction of $\texttt{Non-metal} \ntwiddle \neg \texttt{Essential}$, and yet $\texttt{Non-metal} \land \texttt{Essential} \twiddle \texttt{Gas}$ is not satisfied by the preferential context. 

What is being described is precisely what rational monotonicity tries to avoid. Fortunately, we can ensure the satisfaction of (RM) by preferential contexts through the restriction that the preference relation on objects be \textit{modular}. 

\begin{lemma}
	\label{definition:modularity}
	If $\prec$ is a partial order on a set $G$ then the following are equivalent: 
    \begin{enumerate}
        \item for any $g_0, g_1, g_2 \in G$, if $g_0 \not \prec g_1$, $g_1 \not \prec g_0$, and $g_2 \prec g_0$, then $g_2 \prec g_1$; 
        \item for any $g_0, g_1, g_2 \in G$, if $g_0 \prec g_1$ then either $g_2 \prec g_1$ or $g_0 \prec g_2$; 
        \item there exists a totally ordered set $\mathbb{N}$ with the strict order $<$ and a ranking function $\mathsf{R} : G\rightarrow\mathbb{N}$ such that $g_0 \prec g_1$ if and only if $\mathsf{R}(g_0) < \mathsf{R}(g_1)$;
    \end{enumerate}
    and if $\prec$ satisfies these conditions then we call it \emph{modular}.
\end{lemma}

A strict partial order that is \textit{modular} is an ordering which stratifies the set of objects such that if two objects are incomparable, then they occupy the same strata \cite{ginsberg1986counterfactuals,what-does-a-conditional-knowledge-base-entail}. That is, if we consider some $g_0, g_1 \in \minDerivation{\phi}$ such that $g_0 \in \minDerivation{\phi \land \psi}$, then there is no $g_2 \in \minDerivation{\phi \land \psi}$ such that  $g_1 \prec g_2$, since modularity requires that $g_0 \prec g_2$. With modularity in mind, we can define a ranked context as

\begin{definition}[Ranked Context]
	\label{definition:ranked-context}
	A \emph{ranked context} $\mathbb{R} = ( G,M, I, \prec )$ is a preferential context where $\prec$ is modular. A corollary of \Cref{definition:modularity} is that $\prec$ can represented by a ranking function $\mathsf{R}$ satisfying convexity. We sometimes abuse notation write a ranked context as $\rcontext$. Then, $g_0 \prec g_1$ if and only if $\mathsf{R}(g_0) < \mathsf{R}(g_1)$. We use the shorthand $\mathbb{R}(g)$ to represent $\mathsf{R}(g)$ where $\mathsf{R}$ is the ranking function for $\mathbb{R}$. 
\end{definition}


We denote the consequence relation induced by a ranked context $\mathbb{R}$ as $\twiddle_\mathbb{R}$. From the definition it is clear that ranked contexts are a subset of preferential contexts, and so every $\twiddle_\mathbb{R}$ is at least preferential consequence relation. Then, all that is required to show that $\twiddle_\mathbb{R}$ is a rational consequence relation is to show that it satisfies (RM). 

\begin{restatable}[Soundness of Ranked Contexts]{theorem}{SoundnessR}
    If $\twiddle_\mathbb{R}$ is the consequence relation induced by a ranked context $\mathbb{R}$, then $\twiddle_\mathbb{R}$ is closed under (RM). 
\end{restatable}

    

The completeness result for ranked contexts with respect to rational consequence relations follows a similar trajectory to the result for preferential contexts: a derived ranked context, $\mathbb{R}^\mathcal{R}$, can be constructed from a ranked interpretation, $\mathcal{R}$, such that for any consequence relation, $\twiddle_\mathcal{R}$, induced by a ranked interpretation there is another consequence relation $\twiddle_{\mathbb{R}}$ induced by its derived context that corresponds to an analogous pattern of reasoning. 


\begin{restatable}{lemma}{DefeasibleCorrespondenceRanked}
    \label{lemma:correspondence-ranked}
    If $\mathcal{R}$ is a ranked interpretation and $\mathbb{R}^\mathcal{R}$ is the derived ranked context, then $\mathcal{R} \Vdash \phi_\mathcal{M} \twiddle \psi_\mathcal{M}$ if and only if $\mathbb{R}^\mathcal{R} \Vdash \phi \twiddle \psi$. 
\end{restatable}

\begin{restatable}[Completeness of Ranked Contexts]{theorem}{CompletenessR}
    A consequence relation is \emph{rational} if and only if the consequence relation induced by a ranked context.
\end{restatable}

\section{Non-monotonic Entailment in Ranked Contexts}
\label{section:rational-fca}

In the previous section, we examined how a formal context can be extended to provide semantics for defeasible conditionals and demonstrated that this extension represents a faithful translation of propositional rational consequence relations.  We now turn to the matter of defeasible entailment within FCA: that is, given a set $\Lambda$ of defeasible conditionals, what else should be inferred? Recall from the discussion in \Cref{subsection:defeasible-reasoning} that Tarskian notions of consequence remain monotonic. Consequently, a notion of entailment based on the set of all ranked contexts satisfying $\Lambda$ fails to capture the non-monotonicity of defeasible inference.

\subsection{Finding Order} 
\label{subsection:finding-order}

As an intermediary step, we propose \texttt{ObjectRank}---a derivation of \Cref{algorithm:baserank}---as an approach to constructing a preference relation over objects, which, as we will show, maintains the desirable properties around minimality that \texttt{BaseRank} ensures. Thus far, we have remained agnostic regarding how any preference relation over objects might be constructed. However, manually constructing a suitable order is not necessarily a trivial task: it is not always obvious when one object should be considered more typical than another. Instead, our proposed approach considers background information that represents a perspective on how the domain is expected to behave.

The background information is encoded as a set $\Delta$ of defeasible conditionals. For instance, \texttt{mammal} $\twiddle$ \texttt{viviparous} represents the expectation that normal objects that are mammals are viviparous. The rank of an object is then determined by how consistent it is with $\Delta$. 

We define the condition of $\Delta$-validity for a context to ensure consistency between the expected behaviour encoded by $\Delta$ and the information present in the context. In essence, this condition requires that for every subset of expected behaviour, there must be a plausible object (i.e., one that is consistent with the subset) that non-vacuously satisfies a conditional in the subset; that is, satisfies the antecedent of at least one conditional. Without this condition, it is not guaranteed that the ranked context derived from \texttt{ObjectRank} satisfies $\Delta$. One immediate consequence of $\Delta$-validity is that there cannot be a conditional $\phi \twiddle \psi \in \Delta$ such that $\phi^\downarrow = \emptyset$. 

\begin{definition}[$\Delta$-Validity]
    \label{definition:delta-valid}
    A formal context $(G,M,I)$ is \emph{$\Delta$-valid} for a set $\Delta$ of defeasible conditionals if and only if for every non-empty subset $\Delta_i \subseteq \Delta$ there exists some $g \in G$ such that $g \models \Delta_i$ and there exists $\phi \twiddle \psi \in \Delta_i$ with $g \in \phi^\downarrow$. 
\end{definition}

\vspace{-1em}
\begin{algorithm2e}
    \caption{\texttt{ObjectRank}}
    \KwIn{A finite set $\Delta$ of defeasible conditionals over $\mathsf{M}$}
    \KwIn{A $\Delta$-valid formal context $(G,M,I)$}
    \KwResult {A partition on the set $G$ of objects $(R_0, \ldots, R_n, n)$}
    $i := 0$\;
    $R_0 := G$\;
    $\Delta_0 = \mathsf{Material}(\Delta)$\;
    \While{$\Delta_i \not= \emptyset$}{
        $R_{i+1} := \{ g \in R_i \mid \exists \; \phi \rightarrow \psi \in \Delta_i \text{ such that } g \not\models \phi \rightarrow \psi \}$\;
        $R_i = R_i \setminus R_{i+1}$\;
        $\Delta_{i+1} := \Delta_i \setminus \{(\phi \rightarrow \psi) \in \Delta_i \mid \exists g \in R_i \text{ such that } g \Vdash \phi \}$\;
        $i := i + 1$\;
    }
    \eIf{$R_i = \emptyset$}{
         $n = i-1$\;  } {
        $n := i$\;
    }
     \Return{ $(R_0, \ldots, R_n, n)$}\;
\end{algorithm2e}

Essentially, \texttt{ObjectRank} can be considered a ranking function as per \Cref{definition:ranked-context}. The process of constructing a ranked context from the partition derived from \texttt{ObjectRank} is intuitive, and results in a pleasing representation which is seen in \Cref{section:example}. From this representation it is quite obvious that the induced preference relation satisfies modularity. In the next subsection we show that the ranked context resulting from \texttt{ObjectRank} preserves the properties of the original \Cref{algorithm:baserank}. 

\begin{proposition}
    If $\mathbb{R}_\texttt{OR}$ is the ranked context derived from the \texttt{ObjectRank} algorithm with $\Delta$ and a $\Delta$-valid context $(G,M,I)$, then $\mathbb{R}_\texttt{OR}$ satisfies $\Delta$.
\end{proposition}

\begin{proof}
    We show by induction on the ranks of $\mathbb{R}_\texttt{OR}$ that $\Delta$ is satisfied. As the base case, we use $\mathbb{R}^{-1}(0) = \{g \in G \mid g \models \Delta \}$ to denote the objects on the $0^{th}$ rank; that $\mathbb{R}^{-1}(0)$ is non-empty is guaranteed by $\Delta$-validity. The set of conditionals in $\Delta$ that are non-vacuously satisfied by objects in $\mathbb{R}^{-1}(0)$ is given by $\Delta_0 = \{\phi \rightarrow \psi \in \Delta \mid \exists g \in \mathbb{R}^{-1}(0) \text{ s.t. } g \in \phi^\downarrow \}$. Again, that $\Delta_0$ is non-empty is guaranteed by $\Delta$-validity. Since $0$ is the lowest rank, and the set of objects $\mathbb{R}^{-1}(0)$ is consistent with $\Delta$, every defeasible counterpart to every implication in $\Delta_0$ is satisfied by the ranked context as a whole.

    We assume that each $\phi_k \rightarrow \psi_k \in \Delta_k$ will be satisfied by all $g \in \bigcup_{i=0}^{k} \mathbb{R}^{-1}(i)$, but that $\mathbb{R}^{-1}(k)$ is the lowest rank containing an object that non-vacuously satisfies $\phi_k \rightarrow \psi_k$, and so the whole context satisfies all the defeasible counterparts to every implication in in $\Delta \setminus \bigcup_{i=0}^{k} \Delta_{i}$. 
    
    For $k+1$, the set of implications under consideration is given by $\Delta_r = \Delta \setminus \bigcup_{i=0}^{k} \Delta_{i}$. By $\Delta$-validity, there exists some $g \in G$ such that $g \models \Delta_{r}$ and $g$ non-vacuously satisfies some $\phi_{k+1} \rightarrow \psi_{k+1} \in \Delta_{r}$. All that needs to be shown is that $g$ has not been assigned a rank lower than $k+1$. Suppose that it had, and that $\mathbb{R}(g) = j < k+1$, then $g \models \Delta \setminus \bigcup_{i=0}^j \Delta_i$ for which it is necessary that $\Delta_r$ is a strict subset of. It then follows that all objects in $\mathbb{R}^{-1}(j)$ would satisfy $\Delta_{r}$. But then, $\phi_{k+1} \rightarrow \psi_{k+1}$ should be an element of $\Delta_j$, and could not be an element of $\Delta_r$, which is a contradiction and $g$ cannot be on a rank lower than $k+1$. Then, $g \in \mathbb{R}^{-1}(k+1)$ and so  $\phi_{k+1} \rightarrow \psi_{k+1} \in \Delta_{k+1}$ is non-vacuously satisfied for the first time on rank $k+1$. It follows that the defeasible counterpart is satisfied by the context as a whole. 
    
\end{proof}

A corollary of $\Delta$-validity is that every iteration of \texttt{ObjectRank} considers a strictly decreasing subset of conditionals. Eventually, the set of conditionals under consideration will be empty, and the algorithm will terminate.

\subsection{Rational Closure}
\label{subsection:rational-closure-fca}

We can now describe the rational closure of a $\Delta$-valid formal context. The first point is to show that, by the same argument as in the propositional case, a preference over ranked contexts can be constructed. 

\begin{definition}[Preference Relation on Contexts]
    \label{definition:preference-over-contexts}
    Let $\mathbb{R}^\Delta$ denote the set of ranked contexts over the same $(G,M,I)$ that satisfy $\Delta$. For two contexts $\mathbb{R}_1, \mathbb{R}_2 \in \mathbb{R}^\Delta$, we write $\mathbb{R}_1 \preceq_\mathbb{R} \mathbb{R}_2$ if and only if $\mathbb{R}_1(g) \leq \mathbb{R}_2(g)$ for all $g \in G$. 
\end{definition}

Ranked contexts that consider objects as typical as possible are preferred, such that we continue to reason classically unless required. The rational closure is then determined by this $\preceq_\mathbb{R}$-minimal context, as it represents the most conservative pattern of reasoning. 

\begin{proposition}
    \label{proposition:object-rank-minimal}
    Given a set $\Delta$ of defeasible conditionals and a $\Delta$-valid formal context $(G,M,I)$, the ranked context $\mathbb{R}^\Delta_\texttt{OR}$ derived from \texttt{ObjectRank} is $\preceq_\mathbb{R}$-minimal in the set $\mathbb{R}^\Delta$.  
\end{proposition}

\begin{proof}
    We show minimality by induction on rankings. Let $\Delta$ be a set of defeasible conditionals, $(G,M,I)$ a $\Delta$-valid context, and $\mathbb{R}_1^\Delta$ an arbitrary ranked context satisfying $\Delta$. 

    For the base case, suppose $\mathbb{R}_\texttt{OR}^\Delta(g) = 0$. Then, if $\mathbb{R}_1^\Delta \preceq_\mathbb{R} \mathbb{R}_\texttt{OR}^\Delta$, it is required that $\mathbb{R}_1^\Delta(g)=0$ too. Now, assume that, for all $k < n$, if  $\mathbb{R}_\texttt{OR}^\Delta(g) = k$, then $\mathbb{R}_1^\Delta(g)=k$. Our aim is then to show that, if  $\mathbb{R}_\texttt{OR}^\Delta(g) = n$, then $\mathbb{R}_1^\Delta(g)=n$. 

    Suppose this is not the case and there exists $g \in G$ such that $\mathbb{R}_1^\Delta(g) < \mathbb{R}_\texttt{OR}^\Delta(g) = n$. By the construction of  $\mathbb{R}_\texttt{OR}^\Delta$, there must exist some $\phi \twiddle \psi \in \Delta$ such that $g \Vdash \phi$ and $g \nVdash \psi$, and for all $h \in \minDerivation{\phi}$ of  $\mathbb{R}_\texttt{OR}^\Delta$, it is the case that  $\mathbb{R}_\texttt{OR}^\Delta(h) = n-1$. That is, there is some defeasible conditional in $\Delta$ for which $g$ is a counter-example, such that the minimal objects satisfying the premise of the conditional have the rank just below $g$ . 

    Recall our assumption, that $\mathbb{R}^\Delta_1(i) = \mathbb{R}^\Delta_\texttt{OR}(i)$ for any $i \in G$ where $\mathbb{R}^\Delta_\texttt{OR}(i) < n$, and so for all $h \in \minDerivation{\phi}$ we have that $\mathbb{R}_1^\Delta(h) = n-1$. Then also, $\mathbb{R}_1^\Delta(g) \leq n-1$. But then $g \in \minDerivation{\phi}$ for $\mathbb{R}_1^\Delta$ and so $\mathbb{R}_1^\Delta$ does not satisfy $\Delta$. 
    
    Therefore $\mathbb{R}_1^\Delta(g) > n-1$ and thus $\mathbb{R}_1^\Delta \not \preceq_\mathbb{R} \mathbb{R}_\texttt{OR}^\Delta$. 
    
\end{proof}

We strengthen the above proposition, and show that \texttt{ObjectRank} constructs a ranked context that is the unique $\preceq_\mathbb{R}$-minimum context. 

\begin{proposition}
    \label{proposition:object-rank-minimum}
    If $\mathbb{R}^\Delta$ is the set of all ranked contexts derived from a formal context $(G,M,I)$ satisfying $\Delta$, and $\mathbb{R}^\Delta_{\texttt{OR}} \in \mathbb{R}^\Delta$ is the ranked context derived from the \texttt{ObjectRank} algorithm, then $\mathbb{R}^\Delta_\texttt{OR}$ is the unique $\preceq_\mathbb{R}$-minimum ranked context. 
\end{proposition}

\begin{proof}
    Let $\mathbb{R}_1^\Delta$ be another $\preceq_\mathbb{R}$-minimal context satisfying $\Delta$. Then $\mathbb{R}_1^\Delta$ and $\mathbb{R}_\texttt{OR}^\Delta$ are incomparable, and there exists some $g_0\in G$ such that $\mathbb{R}_1^\Delta(g_0) < \mathbb{R}_\texttt{OR}^\Delta(g_0)$. Let $\mathbb{R}_\texttt{OR}^\Delta(g_0) = n$. Then, by construction of $\mathbb{R}_\texttt{OR}^\Delta$ there exists some $\phi \twiddle \psi \in \Delta$ where $g_0 \Vdash \phi$ and $g_0 \nVdash \psi$, and that the minimal objects $h \in \minDerivation{\phi}$ (with respect to $\mathbb{R}_\texttt{OR}^\Delta$) have the rank $n-1$.

    But then, in order for $\mathbb{R}_1^\Delta$ to satisfy $\Delta$ there must be another object $g_1 \in G$ such that $g_1 \Vdash \phi$ and $g_1 \Vdash \psi$ with $\mathbb{R}_1^\Delta(g_1) < \mathbb{R}_1^\Delta(g_0) \leq n-1$. It follows that $\mathbb{R}_1^\Delta(g_1) < \mathbb{R}_\texttt{OR}^\Delta(g_1)$ and so there must be another defeasible conditional that $g_1$ is an exception to. 

    Repeating the same argument finitely many times, we can show that for any natural number $i\leq n$ there must exist $g_i\in G$ such that $\mathbb{R}_1^\Delta(g_i)<\mathbb{R}^\Delta_\texttt{OR}(g_i)=n-i$. In particular, since $n$ is finite, this includes $i=n$. However in this case we must have some $g_n$ such that $R(g_n)<\mathbb{R}_\texttt{OR}(g_n)=0$, which is a contradiction since by definition $\mathbb{R}_1^\Delta(g_n)\geq 0$. Thus, $\mathbb{R}^\Delta_\texttt{OR}$ is the unique minimum ranking function in the set of ranking functions over $(G,M,I)$ which satisfy $\Delta$.
    
\end{proof}

There is a subtle observation that has thus far been omitted from the discussion. The aim of this section was to show that a faithful construction of rational closure in FCA is possible. We note that where, in the propositional case, the \texttt{BaseRank} algorithm considers a complete set of valuations for a knowledge base $\Delta$, in \texttt{ObjectRank}, we only consider the objects of a formal context. It is not a guarantee that the set of object intents are representative of the entire attribute space. With this in mind, we define \textit{contextual rational closure} as

\begin{definition}[Contextual Rational Closure]
    \label{definition:contextual-rational-closure}
    Given a ranked context $\mathbb{R}_\texttt{OR}^\Delta$, obtained through \texttt{ObjectRank} applied to $\Delta$ and a formal context $(G,M,I)$, we state that $\phi \twiddle \psi$ is in the \emph{contextual rational closure} of $\mathbb{R}_\texttt{OR}^\Delta$ if and only if $\mathbb{R}_\texttt{OR}^\Delta$ satisfies $\phi \twiddle \psi$. We denote this by $\mathbb{R}_\texttt{OR}^\Delta \dentails_{CRC} \phi \twiddle \psi$.
\end{definition}

That the definition of contextual rational closure describes a rational consequence relation follows immediately from the fact that it is the consequence relation described by some ranked context; and so, on both the object and meta level, ranked contexts satisfy the rationality postulates. 
\section{Discussion}\label{section:discussion}

We provide the following longer example as a demonstration of the \texttt{BaseRank} algorithm, as well as how the rational closure entailment relation in FCA differs from the propositional definition.

\begin{example}
\label{section:example}
The context describes a group of people $\{$\texttt{alice}, \texttt{bob}, ... ,\texttt{frank}$\}$ and corresponding attributes indicating friendship with a person. The (classical) implication $\texttt{fw. alice} \rightarrow \texttt{fw. bob}$ over attributes should be interpreted as ``Those who are friends with Alice are friends with Bob''. In turn, the defeasible counterpart $\texttt{fw. alice} \twiddle \texttt{fw. bob}$ describes the scenario where the typical friends of Alice are friends with Bob. 
\vspace{-1em}
\begin{figure}[H]
    \centering
    \begin{subfigure}[t]{0.4\textwidth}
        \centering
        \caption{A ranked context describing friendship among a set of individuals, derived from \texttt{ObjectRank}}
        \label{subfig:context}
        \begin{tabular}{|c|l|l|l|l|l|l|l|}
            \hline
            $\mathsf{R}$ &                  & \rotatebox{90}{\texttt{fw. alice}} & \rotatebox{90}{\texttt{fw. bob}} & \rotatebox{90}{\texttt{fw. charlie }} & \rotatebox{90}{\texttt{fw. david}} & \rotatebox{90}{\texttt{fw. eva}} & \rotatebox{90}{\texttt{fw. frank}} \\
            \hline
            \multicolumn{6}{c}{} \vspace{-0.9em} \\ \hline
            \texttt{0} & \texttt{bob}     &                                    & $\times$                         & $\times$                             & $\times$                           &                                  &                                    \\ \hline
                       & \texttt{eva}     &                                    & $\times$                         &                                      &                                    & $\times$                         &                                    \\ \hline
            \texttt{1} & \texttt{charlie} &                                    &                                  & $\times$                             &                                    &                                  &                                    \\ \hline
                       & \texttt{frank}   & $\times$                           & $\times$                         & $\times$                             &                                    &                                  &                                    \\ \hline
            \texttt{2} & \texttt{alice}   & $\times$                           &                                  & $\times$                             &                                    & $\times$                         & $\times$                           \\ \hline
                       & \texttt{david}   & $\times$                           &                                  &                                      & $\times$                           & $\times$                         & $\times$                           \\ \hline
            \specialrule{1.25pt}{0pt}{0pt}
        \end{tabular}
    \end{subfigure}
    \hfill
     \begin{minipage}[t]{0.48\textwidth}
        \begin{subfigure}[t]{\textwidth}
            \caption{A set $\Delta$ of defeasible conditionals describing plausible friendship relations, used for the \texttt{ObjectRank} algorithm}
            \label{subfig:delta}
            \[
              \Delta =
              \left\{
              \begin{aligned}
                   & \texttt{fw. alice} \twiddle \texttt{fw. bob}, \\[0.3em]
                   & \texttt{fw. charlie} \twiddle \texttt{fw. david}, \\[0.3em]
              \end{aligned}
              \right\}
            \]
        \end{subfigure}
        
        \vspace{3em}  
        \begin{subfigure}[t]{\textwidth}
            \captionsetup{justification=raggedright, singlelinecheck=false}
            \caption{Rational Entailment from $\mathbb{R}_\texttt{OR}^\Delta$}
            \label{subfig:new}
              \(\mathbb{R}_\texttt{OR}^\Delta \dentails_\texttt{CRC} \texttt{fw. david} \twiddle \texttt{fw. charlie}\), but \\ \(\mathbb{R}_\texttt{OR}^\Delta \ndentails_\texttt{CRC}\texttt{fw. david$\;\land$fw. eva} \twiddle \texttt{fw. charlie}\)
        \end{subfigure}       
    \end{minipage}
    \label{fig:example-context-delta}
    \caption{Illustration of contextual rational closure.}
\end{figure}
\vspace{-1em}
\end{example}

That the contextual rational closure entailment relation is indeed non-monotonic, observe that in \Cref{subfig:context}---which defines the the contextual rational closure of $(G,M,I)$ and $\Delta$--- $\mathbb{R}^\Delta_\texttt{OR} \dentails \texttt{fw. eva} \twiddle \texttt{fw. bob}$ holds. If the defeasible conditional $\texttt{fw. eva} \twiddle \texttt{fw. frank}$ were added to $\Delta$, the new ordering derived from \texttt{ObjectRank} would be the same for all objects except for \texttt{eva}, which would occupy the newly most exceptional rank, $3$. The entailment $\texttt{fw. eva} \twiddle \texttt{fw. bob}$ no longer holds retracted, while we gain a new consequence: $\mathbb{R}^\Delta_\texttt{OR} \dentails \texttt{fw. eva} \twiddle \texttt{fw. alice}$.

This example provides a good representation of the distinction between rational closure in the propositional case, where the all valuations are present, and our more specialised construction in FCA. If the set of objects in \Cref{fig:context} were extended so as to become representative of the entire attribute space, every defeasible conditional in $\Delta$ would be non-trivially satisfied on the $0^{th}$ rank: the object with every attribute would satisfy every conditional in $\Delta$. This is not a problem per se; but rather illustrates how reasoning with respect to a formal context allows one to restrict consideration to specific information is more appropriate. 
\vspace{-1em}
\section{Related Works}\label{section:related-works}
As mentioned in the introduction, the inclusion of defeasibility into FCA represents a relatively new line of research. This paper describes a further development of the ideas initially presented by Carr et al. \cite{carr2024non}. Independently of this work, Ding et al. \cite{defeasibleReasoningOnConcepts} and Liang et al. \cite{KLM-style-concepts} developed an approach to KLM-style defeasible reasoning in FCA. While their work falls within a similar domain, it focuses on developing a defeasible relation between concepts (i.e., allowing exceptions in the sub/super-concept relation), whereas our work discusses defeasible attribute implications. Moreover, their approach is limited to \textit{cumulative consequence relations}, which do not support negation or disjunction, whereas we explore preferential and rational consequence relations. 
\section{Conclusions \& Future Work }
\label{section:conclusion}

In this paper, we provide an approach to defeasible reasoning in FCA that serves as a faithful reconstruction of the KLM framework, originally formulated for propositional logic. This facilitates the discovery of \textit{defeasible} dependencies between sets of attributes in FCA, modelled as implication-like defeasible conditionals. These dependencies are characterised by rational consequence relations, and follow a describable pattern of reasoning. 

We extend this result to the more abstract notion of non-monotonic entailment; we provide an interpretation of rational closure for the FCA setting. Our approach diverges from the original definition in the sense that a formal context allows consideration to be given only to the objects in a particular context, rather than the set of all possible objects; the latter scenario being a more direct translation of rational closure. This yields an entailment relation based on information pertinent to the domain, omitting possible, but not present information from consideration.

Future avenues of interest include developing a perspective on \textit{typical concepts}---which incorporate the preference relation on objects in their construction---and their corresponding concept lattices.An early notion of this was introduced in a previous paper \cite{carr2024non}, but there is much room for further investigation. More immediately, it would certainly be useful to investigate complexity results of contextual rational closure. Moreover, rational closure is one of many existing notion of defeasible entailment; exploring others, such as \textit{lexicographic} \cite{lehmann2002perspectivedefaultreasoning} and \textit{relevant closure} \cite{casisini_relevant}, may also prove interesting. 

\bibliographystyle{splncs04}
\bibliography{references}

\begin{thebibliography}{10}
\providecommand{\url}[1]{\texttt{#1}}
\providecommand{\urlprefix}{URL }
\providecommand{\doi}[1]{https://doi.org/#1}

\bibitem{armstrong1974dependency}
Armstrong, W.: Dependency structure of data base relationships. Proc. IFIP Congress pp. 580--583 (1974)

\bibitem{carr2024non}
Carr, L., Leisegang, N., Meyer, T., Rudolph, S.: Non-monotonic extensions to formal concept analysis via object preferences. In: Southern African Conference for Artificial Intelligence Research. pp. 476--492. Springer (2024)

\bibitem{casisini_relevant}
Casini, G., Meyer, T., Moodley, K., Nortj{\'e}, R.: Relevant closure: A new form of defeasible reasoning for description logics. In: Ferm{\'e}, E., Leite, J. (eds.) Logics in Artificial Intelligence. pp. 92--106. Springer International Publishing, Cham (2014)

\bibitem{casini2019taking}
Casini, G., Meyer, T., Varzinczak, I.: Taking defeasible entailment beyond rational closure. In: European Conference on Logics in Artificial Intelligence. pp. 182--197. Springer (2019)

\bibitem{defeasibleReasoningOnConcepts}
{Ding}, Y., {Manoorkar}, K., {Wayan Switrayni}, N., {Wang}, R.: {Defeasible Reasoning on Concepts}. arXiv e-prints arXiv:2409.04887 (Sep 2024). \doi{10.48550/arXiv.2409.04887}

\bibitem{freund1998preferential}
Freund, M.: Preferential reasoning in the perspective of poole default logic. Artificial Intelligence  \textbf{98}(1-2),  209--235 (1998)

\bibitem{ganter2016conceptual}
Ganter, B., Obiedkov, S.: Conceptual exploration. Springer (2016)

\bibitem{ganter1999contextual}
Ganter, B., Wille, R.: Contextual attribute logic. In: International Conference on Conceptual Structures. pp. 377--388. Springer (1999)

\bibitem{ganter2024formal}
Ganter, B., Wille, R.: Formal Concept Analysis -- Mathematical Foundations. Springer Cham, 2 edn. (2024). \doi{10.1007/978-3-031-63422-2}, \url{https://doi.org/10.1007/978-3-031-63422-2}

\bibitem{ginsberg1986counterfactuals}
Ginsberg, M.L.: Counterfactuals. Artificial Intelligence  \textbf{30}(1),  35--79 (1986)

\bibitem{giordano2015semantic}
Giordano, L., Gliozzi, V., Olivetti, N., Pozzato, G.: Semantic characterization of rational closure: From propositional logic to description logics. Artificial Intelligence  \textbf{226},  1--33 (2015)

\bibitem{KLM2002nonmonotonic}
Kraus, S., Lehmann, D., Magidor, M.: Nonmonotonic reasoning, preferential models and cumulative logics. Artificial intelligence  \textbf{44}(1-2),  167--207 (1990)

\bibitem{lehmann2002perspectivedefaultreasoning}
Lehmann, D.: Another perspective on default reasoning (2002), \url{https://arxiv.org/abs/cs/0203002}

\bibitem{what-does-a-conditional-knowledge-base-entail}
Lehmann, D., Magidor, M.: What does a conditional knowledge base entail? Artificial intelligence  \textbf{55}(1),  1--60 (1992)

\bibitem{KLM-style-concepts}
Liang, F., Manoorkar, K., Palmigiano, A., Tzimoulis, A.: {KLM}-style defeasible reasoning on concepts. In: Napoli, A., Rudolph, S. (eds.) FCA4AI 2024 What can FCA do for Artificial Intelligence 2024. pp. 9--14. CEUR-WS (2024)

\bibitem{pearlSystemZ}
Pearl, J.: System z: a natural ordering of defaults with tractable applications to nonmonotonic reasoning. In: Proceedings of the 3rd Conference on Theoretical Aspects of Reasoning about Knowledge. p. 121–135. TARK '90, Morgan Kaufmann Publishers Inc., San Francisco, CA, USA (1990)

\bibitem{perez2021new}
P{\'e}rez-G{\'a}mez, F., Cordero, P., Enciso, M., Mora, A.: A new kind of implication to reason with unknown information. In: International Conference on Formal Concept Analysis. pp. 74--90. Springer (2021)

\bibitem{rodriguez2014negative}
Rodriguez-Jimeneza, J., Corderoa, P., Encisoa, M., Moraa, A.: Negative attributes and implications in formal concept analysis. Procedia Computer Science  \textbf{31},  758--765 (2014)

\end{thebibliography}
\clearpage
\section*{Appendix}
\label{section:appendix}


\Equivalence*

\begin{proof}
    We begin with the \textit{if} part and assume $\minDerivation{\phi} \subseteq \psi$. For an object $g \in \minDerivation{\phi}$ it is clear that $g \in \phi^\downarrow \cap \psi^\downarrow$. We assume that $g \not \in \underline{\phi^\downarrow \cap \psi^\downarrow}$ and thus the existence of some $h \in \phi^\downarrow \cap \psi^\downarrow$ with $h \prec g$. By basic set theory it should follow that $h \in \phi^\downarrow$, but this causes a contradiction since $g$ is minimal in $\phi^\downarrow$ but $h \prec g$. So no such $h$ can exist, which implies $g \in \underline{\phi^\downarrow \cap \psi^\downarrow}$. Then, we assume $g \in \underline{\phi^\downarrow \cap \psi^\downarrow}$. If $g \not \in \minDerivation{\phi}$ there exists some $h \in \minDerivation{\phi}$ such that $h \prec g$, but we have just shown $\minDerivation{\phi} \subseteq \underline{\phi^\downarrow \cap \psi^\downarrow}$. So $h \in \underline{\phi^\downarrow \cap \psi^\downarrow}$, which is a contradiction since $g \in \underline{\phi^\downarrow \cap \psi^\downarrow}$. Therefore, $\underline{\phi^\downarrow \cap \psi^\downarrow} \subseteq \minDerivation{\phi}$   

    The \textit{only if} part of the proof is trivial. 
    
\end{proof}

\SoundnessP*

\begin{proof}
	The proofs for (REF), (RW), and (AND) are quite obvious and thus omitted. (LLE) is also obvious when we recognise that two attributes are equivalent when they have the same extension. 
    For the (OR) postulate, we have that $\phi \twiddle \gamma$ and $\psi \twiddle \gamma$, and so $\minDerivation{\phi} \subseteq \gamma^\downarrow$ and $\minDerivation{\psi} \subseteq \gamma^\downarrow$. For some $g \in \underline{\phi^\downarrow \cup \psi^\downarrow}$ it follows that $g \in \phi^\downarrow$ or $g \in \psi^\downarrow$ (or both), and that there cannot be some $h \in \phi^\downarrow \cup \psi^\downarrow$ such that $h \prec g$.  If $g \in \phi^\downarrow$ then $g \in \minDerivation{\phi}$, and similarly if $g \in \psi^\downarrow$ then $g \in \minDerivation{\psi}$.  In either case, by our assumption it follows that $g \in \gamma^\downarrow$.  Since $g$ was an arbitrary, we have that $\underline{\phi^\downarrow \cup \psi^\downarrow} \subseteq \gamma^\downarrow$, which is precisely $\phi \lor \psi \twiddle \gamma$. 
    To show (CUT) we assume $\minDerivation{\phi} \subseteq \psi^\downarrow$ and $\underline{\phi^\downarrow \cap \psi^\downarrow} \subseteq \gamma^\downarrow$. By \Cref{lemma1} we have $\minDerivation{\phi} = \underline{\phi^\downarrow \cap \psi^\downarrow}$. Then from equality we get $\minDerivation{\phi} \subseteq \gamma^\downarrow$ which is the condition for $\phi \twiddle_\mathbb{P} \gamma$. 
    The proof for (CM) is similar to (CUT). We assume $\minDerivation{\phi} \subseteq \psi^\downarrow$ and $\minDerivation{\phi} \subseteq \gamma^\downarrow$. Once again we use \Cref{lemma1} and get $\minDerivation{\phi} = \underline{\phi^\downarrow \cap \psi^\downarrow}$. Then, by set equality $\underline{\phi^\downarrow \cap \psi^\downarrow} \subseteq \gamma^\downarrow$ which gives us $\phi \land \psi \twiddle \gamma$.
    
\end{proof}

\Correspondence*

\begin{proof}
    By construction of a derived context, an object $g_s \Vdash m$ if and only if the corresponding state satisfies the reciprocal propositional atom, $l(s) \Vdash m_\mathcal{M}$. That this correspondence holds for conjunction is clear. If $l(s) \Vdash \neg m_\mathcal{M}$ then we have $l(s) \nVdash m_\mathcal{M}$ and by definition of a derived context, $g_s \nVdash m$; it then follows that $g \in G \setminus m^\downarrow$ and so $g \Vdash \neg m$. For disjunction, let $l(s) \Vdash m_\mathcal{M} \lor n_\mathcal{M}$ then it needs to be shown that $g_s \in m^\downarrow \cup n^\downarrow$. Let $l(s)\Vdash m_\mathcal{M}$, then, from before, $g_s \Vdash m$ and $g_s \in m^\downarrow \cup n^\downarrow$. The same argument applies to the case where $l(s) \Vdash n_\mathcal{M}$. 
    
\end{proof}

\DefeasibleCorrespondence*

\begin{proof}
    Let $\mathbb{P}^{\mathcal{P}}=(G,M,I, \prec_*)$ and assume for some $g_s\in G$ assume $g_s\notin \underline{\alpha^{\downarrow}}$. Equivalently, either $g_s\notin\alpha^{\downarrow}$ or $g_s\in\alpha^{\downarrow}$ and there exists some $g_t\in G$ with $g_t\prec_* g_s$. In the first case $l(s)\notin \llbracket\alpha_\mathcal{M}\rrbracket$. In the second case, this is equivalent to $l(s)\in\llbracket\alpha_\mathcal{M}\rrbracket$, and there exists $t\in S$ such that $t\prec s$ and $l(t)\Vdash \alpha_\mathcal{M}$. That is, $g_s\notin \underline{\alpha^{\downarrow}}$ iff. $s\notin \min_{\prec}\llbracket\alpha_\mathcal{M}\rrbracket$.\\
    
    Then $\phi_{\mathcal{M}}\twiddle\psi_{\mathcal{M}}$, iff. $\min_{\prec}\llbracket\phi_\mathcal{M}\rrbracket\subseteq \llbracket\psi_{\mathcal{M}}\rrbracket$. By the previous paragraph $g_s\in \underline{\phi^{\downarrow}}$ iff. $s\in \min_{\prec}\llbracket\phi_\mathcal{M}\rrbracket$, and from Lemma \ref{lemma:correspondence-propositional-compound} we have that $g_s\in\psi^\downarrow$ iff. $s\in\llbracket\psi_{\mathcal{M}}\rrbracket$. Therefore, $\min_{\prec}\llbracket\phi_\mathcal{M}\rrbracket\subseteq\llbracket\psi_{\mathcal{M}}\rrbracket$ iff. $\underline{\phi^\downarrow}\subseteq \psi^\downarrow$, which is equivalent to $\mathbb{P}^{\mathcal{P}}$ satisfying $\phi\twiddle\psi$.
    
\end{proof}

\CompletenessP*

\begin{proof}
    By \Cref{theorem:propositional-completeness} we have that a consequence relation $\twiddle_\mathcal{P}$ is preferential if and only if it the consequence defined by some preferential interpretation $\mathcal{P}$. We omit the proof for this as it is a well known result from \cite{KLM2002nonmonotonic}. Then, by the correspondence result between preferential interpretations and preferential contexts, given by \Cref{lemma:correspondence-propositional-compound} and \Cref{lemma:correspondence-conditionals-propositional-preferential}, we have that a consequence relation $\twiddle_\mathcal{P}$ is preferential if and only if it is the consequence relation defined by a preferential context $\mathbb{P}$. 
    
\end{proof}

\SoundnessR*

\begin{proof}
    We assume that $\phi \twiddle_\mathbb{R} \psi$ and $\phi \ntwiddle_\mathbb{R} \neg \gamma$. By the second assumption there exists at least one object $g \in \minDerivation{\phi}$ such that $g \in \gamma^\downarrow$. It follows that $g \in \underline{\phi^\downarrow \cap \gamma^\downarrow}$. Any other $h \in \underline{(\phi^\downarrow \cap \gamma^\downarrow)}$ is incomparable to $g$ and so there can be no $i \in \phi^\downarrow$ such that $i \prec h$, as then modularity would require that $i \prec g$, but $g \in \minDerivation{\phi}$. Thus $h \in \minDerivation{\phi}$. Then by our first assumption it follows that $h \in \psi^\downarrow$. All elements in $\underline{\phi^\downarrow \cap \gamma^\downarrow} \subseteq \psi^\downarrow$.
    
\end{proof}

\DefeasibleCorrespondenceRanked*

\begin{proof}
    Let $\mathbb{R}^\mathcal{R} = (G,M,I,\prec_*)$ be a derived ranked context where $\mathsf{R}_*$ is the ranking function that induces $\prec$. For some $g_s \in G$ assume $g_s \not \in \minDerivation{\phi}$. Then either $g_s\not \in \phi^\downarrow$ or $g_s \in \phi^\downarrow$ and there exists some $g_t \in \phi^\downarrow$ with $\mathsf{R}_*(g_t) < \mathsf{R}_*(g_s)$. In the first case, $l(s) \not \in \llbracket \phi_\mathcal{M}\rrbracket$. The second case is equivalent to $l(s) \in \llbracket\phi_\mathcal{M}\rrbracket$ and there exists some $t \in S$ such that $\mathsf{R}(t) < \mathsf{R}(s)$ with $t \Vdash \phi$. 

    Then, $\phi_\mathcal{M} \twiddle \psi_\mathcal{M}$ if and only if for every minimal state $v \in \llbracket\phi_\mathcal{M}\rrbracket$ it holds that $v \in \llbracket \psi_\mathcal{M} \rrbracket$. Then, some object $g_s \in \minDerivation{\phi}$ if and only if $s \in \llbracket \phi_\mathcal{M} \rrbracket$, and from \Cref{lemma:correspondence-propositional-compound} it holds that $g_s \in \psi^\downarrow$ if and only if $s \in \llbracket \psi_\mathcal{M} \rrbracket$. Therefore, $\min_{\prec}\llbracket\phi_\mathcal{M}\rrbracket\subseteq\llbracket\psi_{\mathcal{M}}\rrbracket$ if and only if $\minDerivation{\phi} \subseteq \psi^\downarrow$ which is equivalent to $\mathbb{R}^\mathcal{R}$ satisfying $\phi \twiddle \psi$. 
    
\end{proof}

\CompletenessR*

\begin{proof}
    The proof follows the same construction as the completeness result for preferential contexts. Lehmann \& Magidor \cite{what-does-a-conditional-knowledge-base-entail} show that a consequence relation is rational if and only if it the consequence relation induced by some ranked interpretation. We show in \Cref{lemma:correspondence-ranked} that for any ranked interpretation a ranked context can be constructed which induces a corresponding consequence relation. It follows that a consequence relation is rational if and only if it is the relation that can be induced by a ranked context. 
    
\end{proof}
\end{document}